\documentclass[runningheads,a4paper]{llncs}

\usepackage{amssymb}
\usepackage{graphicx}

\usepackage{bm}
\usepackage{url}
\usepackage{array}
\usepackage{tabularx}
\usepackage{amsmath}
\usepackage{tikz}
\usepackage{multirow}

\usetikzlibrary{arrows,snakes,shapes}

\newcommand{\norm}[2]{|| #1 ||_{#2}}

\usepackage{url}
\urldef{\mailsa}\path|{freddie.astrom,michael.felsberg, george.baravdish,claes.lundstrom}@liu.se|

\begin{document}

\mainmatter  

\title{Targeted Iterative Filtering}


%



\author{Freddie {\AA}str{\"o}m\inst{1,2}\and Michael Felsberg\inst{1,2}\and \\ George Baravdish\inst{3} \and Claes Lundstr{\"o}m\inst{2,4}}


\authorrunning{Targeted Iterative Filtering}


\institute{Computer Vision Laboratory, Link{\"o}ping University, Sweden \and
Center for Medical Image Science and Visualization, Link{\"o}ping University, Sweden\and
Department of Science and Technology, Link{\"o}ping University, Sweden \and
Sectra AB, Sweden  \\
\mailsa} 


\toctitle{Targeted Iterative Filtering}
\tocauthor{{\AA}str{\"o}m et al.}

\maketitle

\renewcommand{\div}{\operatorname{div}}
\newcommand{\paraheading}[1]{\vspace*{2mm}\noindent \textbf{#1}}

\newcolumntype{M}{>{\centering\arraybackslash}m{\dimexpr.25\linewidth-2\tabcolsep}}
\renewcommand{\labelitemi}{$\bullet$}

\begin{abstract}
  The assessment of image denoising results depends on the respective application area, \emph{i.e.} image compression, still-image acquisition, and medical images require entirely different behavior of the applied denoising method. In this paper we propose a novel, nonlinear diffusion scheme that is derived from a linear diffusion process in a value space determined by the application. We show that application-driven linear diffusion in the transformed space compares favorably with existing nonlinear diffusion techniques.    

\end{abstract}


\section{Introduction}
Many image processing techniques such, as denoising algorithms, aim to improve the quality of images. Naturally, the definition of quality is dependent on the situation where the images are used. The focus of this work is on denoising algorithms and our approach concentrates on that noise that actually will be visible to an observer, rather than data noise in general. 

A widely applied denoising technique was introduced by Perona and Malik \cite{Perona90scale-spaceand} who proposed a nonlinear partial differential equation (PDE) diffusion scheme. It extends the linear diffusion scheme which is based on the image gradient $\nabla u$ with an edge stopping function $g(|\nabla u|)$ \emph{i.e.}
\begin{center}
\begin{tabular}{c m{2cm}c}
  $\nabla u$ & \centering $\rightarrow$ & $ g(|\nabla u|) \nabla u$ \\
  Linear     &                    & Perona and Malik 
\end{tabular}
\end{center}
where a modification of the diffusion speed is based on the value domain of the image gradient. Another PDE model which has received much attention in recent years is the tensor-based diffusion scheme of Weickert~\cite{Weickert96anisotropicdiffusion}. These diffusion models require the determination of parameters often estimated from the input data. Thus the performance of these methods depend on  the accuracy of the parameter estimation. A particular problem is that image structure of different scale can be present within the same value ranges, hence spatially varying contrast parameters are required. 


In this work we show that by the use of an application dependent transformation to the input data space given by a function $m(u)$, we obtain a nonlinear diffusion formulation. The novel formulation modifies the value domain of $u$ rather than the gradient domain as done in Perona and Malik diffusion \emph{i.e.}
\begin{center}
\begin{tabular}{c m{2cm}c c}
  $\nabla u$ & \centering $\rightarrow$ & $ \nabla m(u)$    \\ 
  Linear     &                               &     Targeted diffusion
\end{tabular}
\end{center}
An energy functional is formulated where the regularization term is expressed using the mapping function $m(u)$ and the resulting Euler-Lagrange equation can be interpreted in terms of nonlinear diffusion. The difference between the edge-stopping function $g$ and the mapping function $m$ is that the former is an data-driven ad-hoc selection whereas $m$ is application-driven.


Image processing tools that target specific regions of an image are relevant in many areas of computer vision, and include high dynamic range imaging~\cite{Debevec:1997:RHD:258734.258884,Dicarlo2000}, infrared imaging~\cite{vollmer2010infrared} and medical imaging~\cite{prokopgalanski2003}. One such region based diffusion filtering method was proposed by Ka\v{c}ur et al.~\cite{Kacur97} who generalized the Perona and Malik diffusion. They model the diffusion PDE with an additional nonlinear function on the range domain from which the gradient is computed. This allows the filtering process to be directed to regions containing particular image structures. Their framework reduces the filtering process in regions determined by the user, but the method still requires the determination of a parameter corresponding to an edge-stopping function within the region of filtering. 

In this work our main contributions are
\begin{itemize}
  \item A novel diffusion scheme is derived by using a mapping function in a variational formulation of standard image diffusion. 
  \item Necessary and sufficient conditions are derived to determine if the solution given by the Euler Lagrange equation yield a minimum of the proposed energy functional.   
  \item We show how the mean and variance of noise present in the signal domain is transformed by the mapping function.
  \item In experiments with computed tomography (CT) images of different noise levels, it is shown that the novel scheme compares favorably to nonlinear scalar diffusion on a data set of 400 images using the structural similarity index \cite{Wang2004}. 
\end{itemize}

\newcommand{\imgsize}{0.18}

\section{Image diffusion}
\subsection{Linear diffusion}
\label{sec:DIFFUSION}
The variational approach to isotropic image diffusion is to minimize the energy functional 
\begin{equation}
  E(u) =  \int_\Omega (u - u^0)^2 \; d {\bm x}  +  \lambda\int_\Omega |\nabla u|^2 \; d {\bm x} \enspace , 
  \label{eq:E_diffusion}
\end{equation}
where $\bm{x} \in \Omega$ and $u^{0}$ denotes the observed image. The constant $\lambda$ is a positive scalar which determines the effect of the regularization. The domain  $\Omega$ is a grid described by the image size in pixels, and $\nabla = \partial_{\bm{x}} = \begin{pmatrix} \partial_{x_1}, & ..., &\partial_{x_n} \end{pmatrix}^T$ is the gradient operator, and $\mathrm{dim}(\nabla) = n$ is the number of dimensions. Other types of regularization terms have previously been investigated \cite{Rudin:1992:NTV:142273.142312,Baravdish11}.
%
%
To minimize $E(u)$, one finds the stationary point $u$ by computing the Euler-Lagrange (E-L) equation
\begin{equation*}
  E_u(u) = 0 \quad \mbox{in} \quad \Omega, \quad \nabla u \cdot \bm{n} = 0 \quad \mbox{on} \quad \partial \Omega \enspace ,
  \label{eq:EL_min}
\end{equation*}
where $\bm{n}$ is the normal vector on the boundary $\partial \Omega$. The E-L equation for~\eqref{eq:E_diffusion} reads
\begin{equation}
  \left \{
    \begin{array}{rlll}
      u-u^{0} - \lambda \Delta u & = 0 & \mbox{in} & \Omega \\
      \nabla u \cdot \bm{n} & = 0 & \mbox{on} &  \partial \Omega \\
    \end{array} \right .
  \label{eq:EL_linear}
\end{equation}
where $\Delta u$ is the Laplacian operator. We solve \eqref{eq:EL_linear} by solving an initial value problem (IVP) and obtain the diffusion equation which has a closed form solution. 


\subsection{Nonlinear diffusion}
Before deriving the proposed diffusion scheme, we define the nonlinear scalar diffusion process of Perona and Malik (PM)~\cite{Perona90scale-spaceand} as 
\begin{equation}
  \left \{
    \begin{array}{rlll}
      u-u^{0} - \lambda \div (g(|\nabla u|) \nabla u) & = 0 & \mbox{in} & \Omega \\
      \nabla u \cdot \bm{n} & = 0 & \mbox{on} &  \partial \Omega \\
    \end{array} \right .
  \label{eq:EL_PM}
\end{equation}
where $g(s) = (1 + ( s/k )^2 )^{-1}$
is a popular choice as the diffusivity function and $k$ is a contrast parameter fixed to suppress the flux at edges and lines in the image. It will be seen that the diffusion process introduced in the subsequent section can be viewed as a nonlinear filter, closely related to PM-diffusion. We solve \eqref{eq:EL_PM} by solving an IVP and obtain the diffusion equation. 

Tensor-based nonlinear diffusion is achieved defining $T = w*\nabla u \nabla u^T$ where $*$ is a convolution operator and $w$ is a Gaussian filter~\cite{Weickert96anisotropicdiffusion,Forstner87,diva2:274026}. Then the diffusion tensor can be computed as $D(T) = O^T g(\Lambda) O \enspace$ where $O$ are the eigenvectors and $\Lambda$ the eigenvalues of $T$~\cite{Fberg2011}. This gives the PDE 
\begin{equation}
  \begin{array}{rlll}
    u-u^{0} - \lambda \div (D(T) \nabla u) & = 0 & \enspace . 
  \end{array}
  \label{eq:EL_AD}
\end{equation}



\section{Targeted iterative filtering}
\label{sec:VED}
 In order to simultaneously consider the signal domain and the application dependent transformation of an image, we express the regularization term of the energy functional \eqref{eq:E_diffusion} in the transformed domain. Let $m (u({\bm x}))$ be a \emph{mapping function} that maps $u({\bm x})$ to its application domain, then define 
\begin{equation}\label{eq:main_func}
  E(u) = \int_\Omega (u - u^0)^2 d{\bm x} + \lambda \int_\Omega |\nabla m (u)|^2 d{\bm x}
\end{equation}
where $m(u) \in {\cal C}^3(\Omega)$ and $\lambda > 0$ is a parameter determining the influence of the regularization term. In the subsequent sections we derive the necessary and sufficient conditions for the functional $E(u)$ to attain a local minimum (for details see supplementary material). 

\subsection{Necessary conditions for local minimum} 
The variational derivative of the the regularization term of $E(u)$ is computed using the G$\hat{\mathrm{a}}$teaux derivative
\begin{align*}
\langle \partial R, v \rangle = \lim_{\varepsilon \rightarrow 0} \frac{|\nabla m (u + \varepsilon v)|^2 - |\nabla m (u)|^2}{\varepsilon} \enspace ,
\end{align*}
where $v \in {\cal{C}}^1(\Omega)$ is an arbitrary function such that $\partial_n v |_{\partial \Omega} = 0$. Using the chain rule $\nabla m(u) = m'(u) \nabla u$ we obtain
\begin{equation*}
\langle \partial R, v \rangle = \lim_{\varepsilon \rightarrow 0} \frac{|\nabla u|^2 (m'(u+\varepsilon v)^2 - m'(u)^2) + m'(u+\varepsilon v)^2 (\varepsilon 2 \nabla u^t \nabla v + \varepsilon^2 |\nabla v|^2)}{\varepsilon} 
\end{equation*}
With Green's identity and Neumann boundary conditions we obtain 
\begin{equation*}
\langle \partial R, v \rangle = \left ( 2 |\nabla u|^2 m'(u)m''(u) - 2 \mathrm{div} (m'(u)^2 \nabla u^t) \right ) v \enspace .
\end{equation*}
Now observe that 
$ \mathrm{div}(m'(u)^2 \nabla u)  =  2 m'(u) m''(u) |\nabla u|^2 + m'(u)^2 \Delta u $.
Using this result, and since $v \neq 0$, the E-L equation reads
\begin{equation}
  \left \{
    \begin{array}{rlll}
      u - u^0 - \lambda(\mathrm{div}(m'(u)^2 \nabla u ) + m'(u)^2 \Delta u) & = 0 & \mbox{in} & \Omega \\
      m'(u)^2 \nabla u \cdot \bm{n} & = 0    & \mbox{on} & \partial \Omega 
    \end{array} \right .
  \label{eq:VEDlinear}
\end{equation}
Since $m'(u)^2 \geq 0 $ it is guaranteed that a solution of \eqref{eq:VEDlinear} exists. Compared to \eqref{eq:EL_PM}, the divergence operator is modulated with the squared steepness of the mapping function. Also, the Laplacian is weighted with the same factor. 
If and only if $m$ is a globally linear function, \eqref{eq:VEDlinear} becomes identical to \eqref{eq:EL_linear}. The difference to nonlinear diffusion is easiest explained in terms of the Lagrangian: Replacing $g(|\nabla u|)$ with $m'(u)^2$ means to replace the robust error function with an intensity dependent factor. 


\subsection{Sufficient conditions for local minimum}
In this section we derive sufficient conditions for the solution of the E-L equation to be a minimum of the regularization term in \eqref{eq:main_func}. The result is summarized in the theorem below. 
%
%
%
We remark that if the mapping function is a strict monotone function, the regularization term in~\eqref{eq:main_func} is obviously convex and the necessary condition is also a sufficient condition. However, in the general case, $m$ is not always a strict monotone function, and this is the case we consider here. 

\begin{theorem}
Let $u^0$ be an observed image in a domain $\Omega \subset \mathbb{R}^2$, 
and denote by $E(u)$ the functional 
\begin{equation*}\label{e5_thm}
E(u) = \int_\Omega (u-u^0)^2\,d\bm{x} + \lambda \int_{\Omega} |\nabla m(u)|^2\,d\bm{x}
\end{equation*}
where $u \in \mathcal{C}^2$ and $m(u) \in \mathcal{C}^3$. Let $\varepsilon > 0$ be arbitrary and consider the set 
$$
B_{\varepsilon} = \left\{h, \nabla h \in L^2(\Omega):\ \norm{h}{L^2(\Omega)}^2 
  \leq \varepsilon^2/C_M,\quad  \norm{\nabla h}{L^2(\Omega)}^2 \geq \varepsilon  \right\}
$$
where 
$$
 C_M = \max_{\bm{x} \in \Omega} \left |[m'(u^*(\bm{x}))m'''(u^*(\bm{x})) -  3(m''(u^*(\bm{x})))^2] |\nabla u^*(\bm{x})|^2 \right | \enspace .
 $$
Then $u^*$ is a local minimum of $E(u)$ given by the solution of the E-L equation (\ref{eq:VEDlinear}) if there exists $\bm{\xi} \in \Omega$ such that 
\begin{equation} \label{d30_thm}
  (m'(u^*(\bm{\xi})))^2 \norm{\nabla h}{L^2(\Omega)}^2 > \varepsilon^2 \enspace ,
\end{equation}
for every $h \in B_{\varepsilon}$.
\end{theorem}


\textbf{Proof}. In order to find the sufficient condition for a minimum, define the regularization term in the functional as
\begin{equation*}\label{e5}
J(u) = \int_{\Omega} |\nabla m(u)|^2\,d\bm{x} = \int_{\Omega} (m'(u))^2 |\nabla u|^2\,d\bm{x}
\end{equation*}

Given a function $\varphi\in{\cal C}^3$, a third order Taylor expansion at the point $0$ is 
\begin{equation*}\label{e15}
\varphi(a) - \varphi(0) =a\varphi'(0) +
\frac{a^2}{2}\varphi''(0) +
\frac{a^3}{6}\varphi'''(a\theta)
\quad 0<\theta<1 \enspace .
\end{equation*}

Let $h\in {\cal C}^1$, then define $\varphi(a) = J(u+ah)$, which determines the first variation $\delta J$ of $J(u)$ as 
$$
\delta J  = \lim_{a\to0}\frac{J(u+ah)-J(u)}{a}=
 \lim_{a\to0}\frac{\varphi(a)-\varphi(0)}{a}=\varphi'(0)
$$
In the same way the second variation $\delta^2 J$ follows. Since $\delta J$ is a linear functional in $h$ and $\delta^2 J$ is a quadratic form in $h$ define $L_1(h)  =\delta J = \varphi'(0)$ and $L_2(h,h) = \delta^2 J = \varphi''(0)$.
%
%
Given that $\varphi$ is differentiable then so is $J$. If $a=1$ then the Taylor expansion is given by
\begin{equation}\label{e20}
J(u(x)+h(x)) - J(u) = L_1(h) + L_2(h,h)  + ||h||^2 \rho(h),
\end{equation}
where $\rho(h)\to0$, as $h\to0$. 

A necessary condition of $u^*$ to be a minimum point of the functional $J(u)$ is
\begin{equation}\label{e30}
\varphi'(0) = L_1(h) = 
 2\int_{\Omega} [ m'(u^*) m''(u^*) |\nabla u^*|^2 h
    + (m'(u^*))^2 \nabla u^* \cdot \nabla h ] \,d\bm{x} = 0
\end{equation}
for every $h$ in a neighborhood of $u^*$. According to the E-L equation the solution $u^*$ must satisfy that 
 \begin{equation}\label{d50} 
   m'(u^*) \neq 0
 \end{equation}
otherwise the trivial solution $J(u^*) = 0$ is obtained. 
Differentiating $\varphi'(a)$ and rewriting the E-L equation using condition~\eqref{d50} obtain $L_2(h,h)$ as  
\begin{equation*} \label{d20}
\frac{1}{2}L_2(h,h) = \int_{\Omega}  [m'(u^*) m'''(u^*) - 3 (m''(u^*))^2] |\nabla u^*|^2 h^2\,d\bm{x} + \int_{\Omega} (m'(u^*))^2 |\nabla h|^2\,d\bm{x}
\end{equation*}
Since $L_2(h,h) > 0$ implies a minimum, we consider the first integral. Assume $m\in{\cal C}^3$ and $u\in{\cal C}^1$, then there is an upper bound $C_M>0$ such that
$$
|[m'(u^*) m'''(u^*) - 3 (m''(u^*))^2] |\nabla u^*|^2  |\leq C_M .
$$
%
Let $\varepsilon > 0$ and $B_{\varepsilon}$ be a set defined by
$$
B_{\varepsilon} = \left\{h, \nabla h \in L^2(\Omega):\ \norm{h}{L^2(\Omega) }^2 
  \leq \varepsilon^2/C_M ,\quad  \norm{\nabla h}{L^2(\Omega)}^2 \geq \varepsilon  \right\}
$$
Given that $h\in B_{\varepsilon}$, then the first integral of $L_2(h,h)$ reads 
\begin{equation*}\label{d30}
\int_{\Omega} [m'(u^*) m'''(u^*) - 3 (m''(u^*))^2] |\nabla u^*|^2 h^2\,d\bm{x}
\geq -C_M \int_{\Omega} h^2\,d\bm{x} \geq -\varepsilon^2.
\end{equation*}
Since $h \in B_{\varepsilon}$ we have
$$
 \int_{\Omega} (m'(u^*))^2 |\nabla h|^2\,d\bm{x} \neq 0 \enspace .
$$
By the mean value theorem of calculus there exists a $\bm{\xi} \in \Omega$ such that $m'(u^*(\bm{\xi})) \neq 0$ and 
$$
 \int_{\Omega} (m'(u^*))^2 |\nabla h|^2\,d\bm{x} = m'(u^*(\xi))^2 \norm{\nabla h}{L^2(\Omega)}^2
$$
Hence 
\begin{align}
\notag L_2(h,h) & \geq 2\int_{\Omega} (m'(u^*))^2 |\nabla h|^2\,d\bm{x} - 2\varepsilon^2 \geq 2 (m'(u^*(\bm{\xi})))^2 \norm{\nabla h}{L^2(\Omega)}^2 - 2\varepsilon^2 \\
\label{d30_old}  & > 2 \varepsilon \, [(m'(u^*(\bm{\xi})))^2 -  \varepsilon \,] > 0
\end{align}
since $h \in B_{\varepsilon}$ and we can always chose $\varepsilon < (m'(u^*(\bm{\xi})))^2$ 
which is the sufficient condition for $u^*$ to be a local minimum of $J(u)$. And the theorem follows. \qed



\section{Noise estimation in the transformed domain}
Due to the nonlinear mapping function, $m$, it is of interest to investigate the transformation of the first and second statistical moments of the input signal. We assume that the image signal can be described by a linear model
\begin{equation*}
  u^0 = u_{0} + \eta \enspace ,
\end{equation*}
where $\eta \sim \mathcal{N}(\mu,\sigma^2)$ and $u^{0}$ is the observed signal, $u_0$ is the noise-free signal and $\eta$ is a noise component normally distributed with mean $\mu$ and variance $\sigma^2$. 
The mean value and the variance are estimated using a second order Taylor series of the mapping function, then the mean value and variance estimates  
\begin{align}
\label{eq:mean-trans}  \hat{\mu}_m & = m(u_0 + \mu) + \frac{1}{2}m''(u_0 + \mu) \, \sigma^2 \\
\label{eq:var-trans}  \hat{\sigma}^2_m  & = \Psi[m](u_0 + \mu) \, \sigma^2 
\end{align}
where
$$
\Psi[m](u_0 + \mu) = m'(u^0)^2 - m(u^0)m''(u^0)
$$
is the energy operator~\cite{springerlink:10.1007/11550518_61}. This shows that the mean value in the transformed domain will depend on the curvature of the transformation used, implying that the mapping will not preserve the average intensity level of the input space. Also that the noise variance in the signal domain is amplified by the energy operator. For complete derivations see supplementary material.

\section{Application to medical imaging}
For the purpose of evaluating the proposed application-driven diffusion scheme we consider the application of medical visualization. We make no claim on superiority over existing techniques in medical visualization, merely we limit ourselves to diffusion methods. The diffusion methods investigated are, the novel targeted filtering scheme (TF), linear diffusion (LD), nonlinear diffusion (PM) and tensor-based image diffusion (AD). 

Visualizations in medical imaging are computed by transfer functions, which usually are piecewise linear~\cite{prokopgalanski2003}. However, sufficiently similar functions produce visualizations that are visually indistinguishable. We use combinations of sigmoid functions, see Fig. \ref{fig:tf_example}, since they are three times continuously differentiable. 
\renewcommand{\imgsize}{0.8}
\begin{figure}[b]
  \centering
  \includegraphics[width=\imgsize\linewidth]{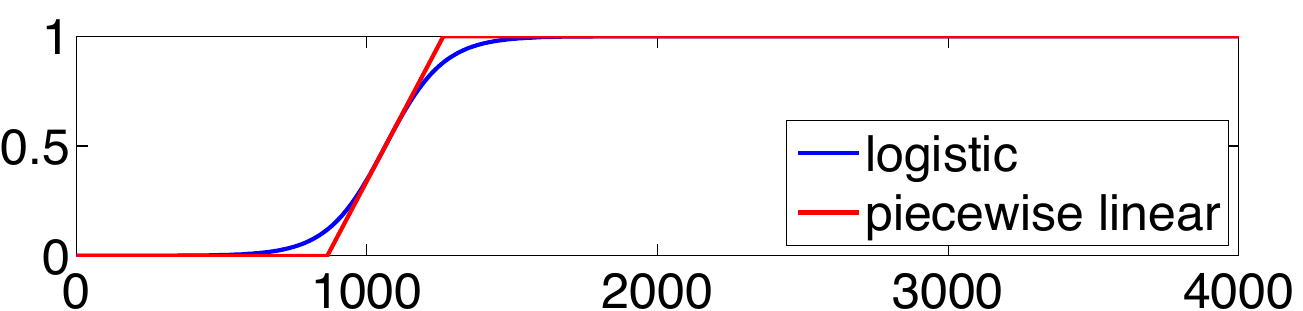}
  \caption{Example of mapping function}
  \label{fig:tf_example}
\end{figure}

\subsection{Selection of mapping function}
Let $m : \mathbb{R}^2 \rightarrow [0, 1]$ be the visualization mapped using a transfer function $m \in {\cal C}^3$ computed from two user defined thresholds $u(\bm x)=u_1$ to $u(\bm x)=u_2$. We define a sigmoid function
\begin{equation}
  m(u(\bm x),a,b)=(1 + \exp(-(u(\bm x)-b)/a))^{-1}\enspace,
  \label{eq:visTF}
\end{equation}
where $a = (u_2-u_1)/4$ is the steepness of the sigmoid function and $b = (u_1+u_2)/2$ defines the offset.
For this choice of mapping function we show that the sufficient condition in~\eqref{d30_thm} is satisfied. 
Then the lower bound of $(m'(u^*))^2$ is given by
\begin{equation*}
  (m'(u^*))^2 = \dfrac{1}{a^2} \dfrac{e^{\frac{2(b-u^*)}{a}}}{(1+ e^{\frac{b-u^*}{a}} )^{4}} 
  \geq \dfrac{1}{a^2} \dfrac{e^{\frac{2(b-1)}{a}}}{(e^{\frac{b}{a}} + e^{\frac{b}{a}} )^{4}} 
  \geq  \dfrac{1}{a^2 16 e^{\frac{2(b+1)}{a}}} 
\end{equation*}
thus the condition \eqref{d30_old} is replaced with $(16 a^2 e^{\frac{2(b+1)}{a}})^{-1} > \varepsilon$.
Details on the determination of the lower bound can be found in the supplementary material. 

\subsection{Numerical aspects}
The derived E-L equation \eqref{eq:VEDlinear} is solved as an IVP problem discretizised using a standard forward Euler scheme. A forward and backward finite difference scheme is used to approximate the image derivatives. 

The derivatives of the mapping function, $m$ are computed analytically. However, before evaluating the derivatives of $m(u)$, the signal $u$ is regularized with a small Gaussian filter. Also to remedy the fact that different propagation speeds are obtained for different slopes of the mapping function, derivatives are normalized to attain a global maximum of $1$. The implementation is available here~\cite{code}. 


\subsection{Experiment setup}
For the evaluation, we add zero mean Gaussian noise to a set of computed tomography (CT) images. The motivation for using additive noise is due to the projection data obtained from the CT scanner contains multiplicative noise. In the CT reconstruction the logarithm of the data is taken, thus multiplicative noise can be modeled as additive noise. All images were scaled to an 8-bit quantisation representation and zero mean Gaussian noise with standard deviation $\sigma = \{5, 10, 15\}$ was added to the test images in the signal domain. According to~\eqref{eq:var-trans}, the noise levels in the visualization domain is $\sigma_{m} = \{51.19, 102.38, 153.36 \}$ using a mapping function with endpoints $u_1 = 864$ and $u_2 = 1264$ where the endpoints are represented using Hounsfield units. 

All diffusion methods were set to iterate the solution until the peak signal to noise (PSNR) value no longer increases. The steplength was set to $0.05$ for all methods except for the proposed method which utilizes the slope of the mapping function as its steplength
$  \lambda = \min(1/(u_2 - u_1)),0.25) $
where $0.25$ is the maximum steplength to ensure stability in the case of linear diffusion~\cite{Weickert96anisotropicdiffusion}. 

The PM and AD contrast parameter was set using the estimated noise levels $\sigma_{est}$ based on~\cite{Fberg2011} and computed according to~\cite{Wolfgang2000} as $k = (e-1)(e-2)^{-1}\sigma_{est}^2$.

 
The peak signal to noise measure (PSNR) and the structural similarity index (SSIM) \cite{Wang2004} was used to evaluate the performance of the proposed algorithm.


\subsection{Results}
\label{sec:2devaluation}
Table~\ref{table:results-woman} shows the SSIM and PSNR values obtained in the visualization domain for a dataset of 400 CT images. 
Comparing the filtering methods with respect to the error measures, then the error values are in favor of the proposed targeted filtering method (TF) higher noise levels. Here it is important to note the fundamental difference between TF and PM. The performance of PM is determined based the estimation of a contrast parameter for the nonlinear mapping function, whereas TF is not. The only parameter required to be determined in TF (as with all iterative methods) is the stopping time to avoid trivial solutions. Thus, disregarding the stopping time, \emph{TF is a non-parametric non-linear diffusion scheme} which behaves similarly to PM diffusion. 

Figure~\ref{fig:seq1_250} and~\ref{fig:seq1_350} visualize the corresponding images of slices 250 with noise level $\sigma = 5$ and 350 with noise level $\sigma = 10$. In addition to the visualizations, respective details are depicted. Visually, the proposed diffusion scheme produces superior results close to edges compared to LD and PM diffusion indicated by the arrows in both figures. LD oversmooths the image and PM simply retains noise close to edges. AD preserves edges well and produces high PSNR and SSIM values but approximately homogeneous regions appear oversmoothed. In Fig.~\ref{fig:seq1_350} it is clear that regions indicated by the arrows have been retained in the proposed method whereas the other diffusion techniques have removed the structure. 

\begin{table}[tb]
  \begin{center}
    \begin{tabular*}{\columnwidth}{@{\extracolsep{\fill}}r cccc  ccc}
      \hline
      \noalign{\smallskip}
      & $\sigma$ & $\hat{\sigma}_m$ & LD & PM & TF & AD \\
      \noalign{\smallskip}
      \hline
      \noalign{\smallskip}
      \multirow{3}{*}{SSIM}  
       & 5 & 51.19 &   0.89 $\pm$ 0.005 &   0.92 $\pm$ 0.004 &   0.93 $\pm$ 0.003 &   0.94 $\pm$ 0.003 \\ 
 & 10 & 102.38 &   0.84 $\pm$ 0.004 &   0.87 $\pm$ 0.005 &   0.89 $\pm$ 0.005 &   0.87 $\pm$ 0.006 \\ 
 & 15 & 153.56 &   0.82 $\pm$ 0.006 &   0.83 $\pm$ 0.005 &   0.86 $\pm$ 0.006 &   0.82 $\pm$ 0.005 
\\
      \noalign{\smallskip}
      \hline
      \noalign{\smallskip}
      \multirow{3}{*}{PSNR} 
       & 5 & 51.19 &  28.44 $\pm$ 0.37 &  30.82 $\pm$ 0.56 &  30.76 $\pm$ 0.51 &  32.18 $\pm$ 0.72 \\ 
 & 10 & 102.38 &  25.92 $\pm$ 0.45 &  27.68 $\pm$ 0.49 &  28.13 $\pm$ 0.53 &  27.88 $\pm$ 0.49 \\ 
 & 15 & 153.56 &  24.81 $\pm$ 0.54 &  25.74 $\pm$ 0.53 &  26.82 $\pm$ 0.59 &  25.38 $\pm$ 0.51  
 \\
      \noalign{\smallskip}
      \hline
    \end{tabular*}
  \end{center}
  \caption{SSIM and PSNR values. $\hat{\sigma}_m$ was computed according to \eqref{eq:var-trans}.}
  \label{table:results-woman}
\end{table}

\renewcommand{\imgsize}{0.23}


\newcommand{\mydrawcircleB}[1]{ %
\begin{tikzpicture}
  \node[anchor=south west,inner sep=0mm] at (0,0) { 
    \includegraphics[width=\imgsize\linewidth]{#1} };
  \draw [red, ultra thick, ->] (0.6,1.5) -- (1.3,1);
\end{tikzpicture}
}

\newcommand{\insertsliceB}[1]{%
  Original & Noisy & Vis. Original & Vis. Noisy \\
  \includegraphics[width=\imgsize\linewidth]{slice#1_vis}  &
  \includegraphics[width=\imgsize\linewidth]{slice#1_noisy_vis}  &
  \mydrawcircleB{slice#1_zoom_vis} &
  \includegraphics[width=\imgsize\linewidth]{slice#1_zoom_noisy_vis}   \\
  LD & PM & TF & AD \\
  \includegraphics[width=\imgsize\linewidth]{slice#1LD_filt_vis}  &
  \includegraphics[width=\imgsize\linewidth]{slice#1PM_filt_vis}  &
  \includegraphics[width=\imgsize\linewidth]{slice#1VED_filt_vis}  &
  \includegraphics[width=\imgsize\linewidth]{slice#1AD_filt_vis} \\
  \hspace*{-1mm}  \mydrawcircleB{slice#1LD_zoom_filt_vis} &
  \hspace*{-1mm}  \mydrawcircleB{slice#1PM_zoom_filt_vis} &
  \hspace*{-1mm}  \mydrawcircleB{slice#1VED_zoom_filt_vis} &
  \hspace*{-1mm}  \mydrawcircleB{slice#1AD_zoom_filt_vis} \\
}

\newcommand{\slice}{250}
\begin{figure}[tbh]
  \begin{center}
    \begin{tabular}{cccc}
      \insertsliceB{\slice}
    \end{tabular}
  \end{center}
  \caption{Slice 250. Noise level $\sigma = 5$. \emph{ Details best viewed on monitor}}
  \label{fig:seq1_250}
\end{figure}


\renewcommand{\slice}{350}

\newcommand{\mydrawcircle}[1]{ %
\begin{tikzpicture}
  \node[anchor=south west,inner sep=0mm] at (0,0) { 
    \includegraphics[width=\imgsize\linewidth]{#1} };
  \draw [red, ultra thick, ->] (2.7,2.1) -- (2.4,2.65);
  \draw [green, ultra thick, ->] (0.9,1.7) -- (0.4,2.1);
\end{tikzpicture}
}

\newcommand{\insertslice}[1]{%
  Original & Noisy & Vis. Original & Vis. Noisy \\
  \includegraphics[width=\imgsize\linewidth]{slice#1_vis}  &
  \includegraphics[width=\imgsize\linewidth]{slice#1_noisy_vis}  &
\hspace*{-1mm}  \mydrawcircle{slice#1_zoom_vis} &
  \includegraphics[width=\imgsize\linewidth]{slice#1_zoom_noisy_vis}   \\
  LD & PM & TF & AD \\
  \includegraphics[width=\imgsize\linewidth]{slice#1LD_filt_vis}  &
  \includegraphics[width=\imgsize\linewidth]{slice#1PM_filt_vis}  &
  \includegraphics[width=\imgsize\linewidth]{slice#1VED_filt_vis}  &
  \includegraphics[width=\imgsize\linewidth]{slice#1AD_filt_vis} \\
  \hspace*{-1mm} \mydrawcircle{slice#1LD_zoom_filt_vis} &
  \hspace*{-1mm}  \mydrawcircle{slice#1PM_zoom_filt_vis} &
  \hspace*{-1mm}  \mydrawcircle{slice#1VED_zoom_filt_vis} &
  \hspace*{-1mm}  \mydrawcircle{slice#1AD_zoom_filt_vis} \\
}

\begin{figure}[tbh]
  \begin{center}
    \begin{tabular}{cccc}
      \insertslice{\slice} 
    \end{tabular}
  \end{center}
  \caption{Slice 350. Noise level $\sigma = 10$. \emph{ Details best viewed on monitor}}
  \label{fig:seq1_350}
\end{figure}

\section{Conclusion}
The performance of image denoising methods has to be assessed with respect to the respective application. In our case, we considered the application of denoising medical images and limit ourselves to diffusion methods. The relevant quality criteria is the result of the visualization after applying a mapping function. We have used the mapping function to derive a novel nonlinear diffusion scheme for targeted iterative diffusion and evaluated the method on a data set of CT images with different noise levels. The proposed method is non-parametric in the sense that it is application-driven rather than data-driven.

\section{Acknowledgment}
This research has received funding from the Swedish Research Council through a grant for the project \emph{Visualization-adaptive Iterative Denoising of Images}, from the ECs 7th Framework Programme (FP7/2007-2013), grant agreement 247947 (GARNICS). From Vinnova project \emph{Online laboratory for medical image analysis}, and Swedish Foundation for Strategic Research, grant SM10-0022. We thank Hanno Scharr at Forschungszentrum J{\"u}lich, Germany, for discussion on the implementation of the anisotropic diffusion scheme.


\newpage

\bibliographystyle{splncs}
\bibliography{egbib}

@inproceedings{diva2:274026,
  author = {Bigun, Josef and Granlund, Gösta H.},
  title = {{Optimal Orientation Detection of Linear Symmetry}},
  Booktitle = {Proceedings of the IEEE First International Conference on Computer Vision},
  year = {1987},
  pages = {433--438},
}

@misc{code,
  author = {{\AA}str{\"o}m, Freddie},
  title = {{Implementation of Targeted Iterative Filtering}},
  howpublished = "\url{http://liu.diva-portal.org/smash/get/diva2:608779/SOFTWARE02}, SSVM'13",
  year = {2013}, 
}

@inproceedings{Forstner87,
 author={F{\"o}rstner, W. and G{\"u}lch, E.},
 title={A fast operator for detection and precise location of distinct points, corners and centres of circular features},
 booktitle={ISPRS Intercommission, Workshop, Interlaken, pp. 149-155.},
 year={1987}
}

@incollection{Wolfgang2000,
   author = {F{\"o}rstner, Wolfgang},
   affiliation = {Bonn University Institute for Photogrammetry Nussallee 15 D-53115 Bonn},
   title = {Image Preprocessing for Feature Extraction in Digital Intensity, Color and Range Images},
   booktitle = {Geomatic Method for the Analysis of Data in the Earth Sciences},
   series = {LNES},
   isbn = {978-3-540-67476-4},
   pages = {165-189},
   volume = {95},
   year = {2000}
}

@incollection {springerlink:10.1007/11550518_61,
   author = {Felsberg, Michael and Jonsson, Erik},
   affiliation = {Computer Vision Laboratory, Linköping University, SE-58183 Linköping, Sweden},
   title = {Energy Tensors: Quadratic, Phase Invariant Image Operators},
   booktitle = {Pattern Recognition},
   series = {LNCS},
   isbn = {978-3-540-28703-2},
   publisher = {Springer},
   pages = {493-500},
   volume = {3663},
   year = {2005}
}

@article{Rudin:1992:NTV:142273.142312,
 author = {Rudin, Leonid I. and Osher, Stanley and Fatemi, Emad},
 title = {Nonlinear total variation based noise removal algorithms},
 journal = {Phys. D},
 issue_date = {Nov. 1, 1992},
 volume = {60},
 number = {1-4},
 month = nov,
 year = {1992},
 issn = {0167-2789},
 pages = {259--268},
 numpages = {10},
 acmid = {142312},
 publisher = {Elsevier Science Publishers B. V.},
}

@book{vollmer2010infrared,
  title={Infrared Thermal Imaging: Fundamentals, Research and Applications},
  author={Vollmer, M. and M{\"o}llmann, K.P.},
  isbn={9783527407170},
  url={http://books.google.se/books?id=5SSiZAMwxtYC},
  year={2010},
  publisher={John Wiley \& Sons}
}

@inproceedings{Dicarlo2000,
  author = {Jeffrey M. DiCarlo and Brian A. Wandell},
  title = {Rendering high dynamic range images},
  booktitle = {Proceedings of the SPIE: Image Sensors 3965},
  pages = {392-401},
  year = {2000}
}

@inproceedings{Debevec:1997:RHD:258734.258884,
 author = {Debevec, Paul E. and Malik, Jitendra},
 title = {Recovering high dynamic range radiance maps from photographs},
 booktitle = {SIGGRAPH '97},
 year = {1997},
 pages = {369--378},
}

@book{prokopgalanski2003,
 author = {Mathias Prokop and Michael Galanski},
 title = {Spiral and Multislice Computed Tomography of the Body},
 year = {2003},
 isbn = {0865118701},
 publisher = {Thieme Verlag},
}

@inproceedings{Baravdish11,
 author={George Baravdish and Olof Svensson},
 title={Image reconstruction with p(x)-parabolic equation},
 booktitle={ICIPE2011, Orlando Florida},
 year={2011}
}

@ARTICLE{Perona90scale-spaceand,
    author = {Pietro Perona and Jitendra Malik},
    title = {Scale-space and edge detection using anisotropic diffusion},
    journal = {IEEE Transactions, PAMI},
    year = {1990},
    volume = {12},
    pages = {629-639}
}

@book{Weickert96anisotropicdiffusion,
    author = {Joachim Weickert},
    title = {Anisotropic Diffusion In Image Processing},
    year = {1998},
    publisher={ECMI Series, Teubner-Verlag, Stuttgart, Germany}
}

@ARTICLE{Wang2004,
author={Zhou Wang and Bovik, A.C. and Sheikh, H.R. and Simoncelli, E.P.},
journal={IEEE Trans}, 
title={Image quality assessment: from error visibility to structural similarity},
year={2004},
month={april} ,
volume={13},
number={4},
pages={600-612},
doi={10.1109/TIP.2003.819861},
ISSN={1057-7149},}

@incollection{Kacur97,
 author = {Jozef Ka\v{c}ur and Karol Mikula},
 title = {Slowed Anisotropic Diffusion},
 year = {1997},
 booktitle = {Scale-Space Theory in Computer Vision},
 series = {LNCS},
 publisher = {Springer Berlin / Heidelberg},
 isbn = {978-3-540-63167-5},
 pages = {357-360},
 volume = {1252}
}

@ARTICLE{Fberg2011,
author={Felsberg, M.},
journal={Image Processing, IEEE Transactions on},
title={Autocorrelation-Driven Diffusion Filtering},
year={2011},
month={july },
volume={20},
number={7},
pages={1797 -1806},
doi={10.1109/TIP.2011.2107330},
ISSN={1057-7149}
}

\end{document}